\documentclass{article}
\usepackage{graphicx}
\usepackage{times}
\usepackage{soul}
\usepackage{url}

\usepackage[utf8]{inputenc}
\usepackage[small]{caption}
\usepackage{graphicx}

\usepackage{amsmath}
\usepackage{amsthm}
\usepackage{booktabs}
\usepackage{algorithm}
\usepackage[noend]{algorithmic}
\usepackage{comment}
\usepackage{balance}
\usepackage{adjustbox}
\usepackage{amssymb}
\usepackage{comment}
\usepackage[]{nicefrac}
\usepackage{ marvosym }
\usepackage{listings}
\usepackage{multicol}
\usepackage{subfigure}
\usepackage{enumerate}
\usepackage{authblk} 

\newcommand{\A}{A} 
\newcommand{\N}{\mathcal{N}} 
\newcommand{\B}{B} 
\renewcommand{\P}{P} 
\newcommand{\Prof}{{\boldsymbol B}} 
\newcommand{\FeasB}{{\mathcal{B}_F}}
\newcommand{\RatB}{{\mathcal{B}_R}}
\newcommand{\R}[2]{\mathcal{R}^{{#1}}_{{#2}}}
\newcommand{\simple}{\textit{simple}}
\newcommand{\weight}{\textit{weight}}
\newcommand{\wswap}{\textbf{\textit{w-swap}}}
\newcommand{\swap}{\textbf{\textit{swap}}}

\newcommand{\CC}{\textbf{\textit{CC}}}

\newcommand{\citep}[1]{\cite{#1}}

 \newtheorem{example}{Example}

 \newtheorem{proposition}{Proposition}
 \newtheorem{definition}{Definition}
 \newtheorem{lemma}{Lemma}

\title{Collective Discrete Optimisation\\ as Judgment Aggregation}
\author[1]{Linus Boes}
\author[2]{Rachael Colley}
\author[2]{Umberto Grandi}
\author[3]{J\'er\^ome Lang}
\author[4]{Arianna Novaro}

\affil[1]{Institut f\"ur Informatik
Heinrich-Heine-Universit\"at D\"usseldorf, Germany. linus.boes@uni.duesseldorf.de}
\affil[2]{IRIT, University of Toulouse, France. \{rachael.colley,umberto.grandi\}@irit.fr}
\affil[3]{CNRS, PSL, France. jerome.lang@lamsade.dauphine.fr}
\affil[4]{Centre d’Economie de la Sorbonne (CES), University of Paris 1 Panth\'eon-Sorbonne, France. arianna.novaro@univ-paris1.fr}

\begin{document}
\maketitle

\begin{abstract}
Many important collective decision-making problems can be seen as multi-agent versions of discrete optimisation problems. Participatory budgeting, for instance, is the collective version of the knapsack problem; other examples include collective scheduling, and collective spanning trees. Rather than developing a specific model, as well as specific algorithmic techniques, for each of these problems, we propose to represent and solve them in the unifying framework of judgment aggregation with weighted issues. We provide a modular definition of collective discrete optimisation (CDO) rules based on coupling a set scoring function with an operator, and we show how they generalise several existing procedures developed for specific CDO problems. We also give an implementation based on integer linear programming (ILP) and test it on the problem of collective spanning trees.
\end{abstract}


\section{Introduction}

Combinatorial discrete optimisation had a tremendous impact on various branches of the economy and society: routing, network design, supply chains, resource and task allocation, navigation and path-finding, scheduling workers, and many more. 
The input of these problems include a description of the objectives to be optimised (in terms of numerical costs and rewards)  with respect to certain constraints, such as: a budget is not exceeded, schedules are consistent, and spanning trees do not contain cycles. 
While in many settings these objectives pertain to a single agent, who is often a central planner, a variety of situations involve a number of agents who have their own objectives which can be conflicting.
A notable example is participatory budgeting (see, e.g., Talmon and Faliszewski~\cite{talmon2019framework}), where a set of projects have to be budgeted given the agents' preferences. Another example is collective scheduling, where a collective decision is to be taken on the order in which projects will be executed and perhaps which projects should be scheduled at all (see, e.g., Pascual et al.~\cite{PascualEtAlAAMAS2018}). 

Judgment aggregation is a unified framework in which collective decisions on a set of binary issues can be obtained from the individual judgments of a set of agents (see, e.g.,  Endriss~\cite{Endriss2016}).
This setting seems to be a suitable candidate to deal with the combinatorial aspects common in discrete optimisation.
Although in its classical formulation all issues have equal weight, a recent paper by  Nehring and Pivato~\cite{nehring2018median} considered weighted issues to quantify their relative importance.
In this paper we define a class of problems named \emph{collective discrete optimisation (CDO) problems}---many of which are multi-agent versions of classical discrete optimisation problems already extensively studied in the literature---and we show that they can all be phrased in the language of judgment aggregation with weighted issues.


A good related example of such generalisation to weighted issues is the one from multi-winner elections to participatory budgeting \cite{talmon2019framework}: both aim at selecting a set of items (candidates or projects) given a maximum capacity, but while multi-winner elections treat candidates uniformly,
projects in participatory budgeting have different costs. 
Thus, participatory budgeting rules can be seen as weighted extensions of multi-winner rules. Similarly, Pascual et al.~\cite{PascualEtAlAAMAS2018} generalise social welfare functions, whose output is a ranking of items, to take into account  tasks' durations for collective scheduling. {\em Our paper shows that there is a common logic behind these generalisations}: we take a more general perspective showing that all these settings can be represented under common definitions of CDO-problems, solved by CDO-rules.

\paragraph{Our Contribution.}
We give a general formulation for collective discrete optimisation problems---i.e., optimisation problems receiving possibly conflicting inputs from multiple agents---and present a modular definition of rules to solve them, based on existing work in judgment aggregation. We also provide equivalence results that show how known judgment aggregation rules correspond to our modular CDO rules (Section~\ref{sec:model}).
We survey a number of CDO problems that were previously studied independently in the literature and show that they can all be represented in our framework, most notably: participatory budgeting, collective scheduling, and collective spanning trees.
Taking one step further, we show how numerous rules defined for specific CDO settings are instances of our modular CDO rules (Section~\ref{sec:CDO}). 
Finally, we present initial experimental results comparing an implementation in integer linear programming (ILP) of three of our rules on the problem of collective spanning trees (Section~\ref{sec:ILP}).



\section{Related Work}\label{sec:related}

There are two main streams of related works. On the one hand, we have papers adapting classical judgment aggregations to deal with weighted issues. On the other, we have papers studying specific collective discrete optimisation settings. 


First, Rey et al.~\cite{rey2020designing} provide efficient and exhaustive embeddings of participatory budgeting problems via DNNF circuits in non-weighted judgment aggregation, giving an initial axiomatic study of asymmetric additive rules extended from known judgment aggregation rules. 
This approach is similar to our own and it proves effective for participatory budgeting; however, it does not generalise easily to other settings. We comment on the use of non-weighted judgment aggregation for collective optimisation in Section~\ref{sec:equivja} by showing the connection to our rules.
%

Nehring and Pivato~\cite{nehring2018median} are the first to introduce and study a setting of judgment aggregation where the issues are weighted, allowing them to model states with different probabilities, criteria of different importance, or simply issues that count more than others. The focus of the authors is the median rule, of which they provide an axiomatisation in the weighted setting. 
We use their definitions to build our general framework for collective discrete optimisation, and we show that  two of our the rules are equivalent to the median rule that they axiomatise. 

Second, to the best of our knowledge, our paper is the first attempt to develop 
 a general-purpose framework to solve a variety of seemingly disconnected problems of collective discrete optimisation, which we survey here. 
The {\em collective selection of weighted projects under a budget} has been studied from numerous perspectives, and a growing literature is focusing on participatory budgeting (see, e.g., the recent survey by Aziz and Shah~\cite{aziz2021participatory}). 
Klamler et al.~\cite{klamler2012committee} study the problem of selecting a committee of size~$k$ from a collective ranking over the candidates, given that candidates have an associated weight (or cost). 
Klamler et al.~\cite{benabbou2016solving} study the multi-agent knapsack problem where the agents' approval sets are elicited incrementally. 
Talmon and Faliszewski~\cite{talmon2019framework} propose nine rules\footnote{Two rules were then shown equivalent by Baumeister et al.~\cite{baumeister2020irresolute}.} for the specific setting of participatory budgeting, by combining different aggregation functions (e.g., greedy and proportional) with different measures of  agent's satisfaction.

Darmann et al.~\cite{DarmannEtAlCOMSOC2008} introduced the problem of adapting voting rules to compute \emph{collective spanning trees} for a set of agents who have preferences over the connections in a given network. Darmann et al.~\cite{darmann2009maximizing} studied the computational complexity of the Borda count 
 to find a spanning tree that maximises the minimum  voter's satisfaction. 
Escoffier et al.~\cite{escoffier2013fair} studied a number of multi-agent combinatorial optimisation problems (including variants of finding a maximum spanning tree), proposing algorithms to find solutions which maximise the minimal utility of an agent. 

Ephrati et al.~\cite{ephrati1993multi} proposed a dynamic iterative search procedure for {\em group planning} aimed at maximising social welfare, where agents do not have to disclose their preferences all at once. 
Klamler and Pferschy~\cite{klamler2007traveling} applied algorithms based on voting rules to find a collective path for agents whose preferences over the edges differ.
Pascual et al.~\cite{PascualEtAlAAMAS2018} studied the combinatorial optimisation problem of {\em collective scheduling}, by proposing the use of rules inspired from social choice theory---such as positional scoring rules, the Kemeny rule, and rules satisfying the Condorcet principle.

\section{The Model}\label{sec:model}

We present here the basic definitions of judgment aggregation with weighted issues and we introduce three families of collective discrete optimisation rules. We then provide equivalence results between our rules and judgment aggregation rules.

\subsection{Basic Definitions}\label{sec:basic_def}
In classical judgment aggregation some agents take a collective decision over a set of possibly interconnected issues, represented as propositional formulas or variables linked by a constraint (cf. the introduction by Endriss~\cite{Endriss2016}). 
In weighted judgment aggregation each issue is associated to a numerical weight (see Nehring and Pivato~\cite{nehring2018median}). 

We thus have a set of $n$ agents (or voters) $\N=\{1,\dots, n\}$ who take a collective decision on which of the $m$ items (or projects, or issues) to accept from the agenda~$A$.
We denote an agenda item as $a \in \A$, and it can denote, for example, a project to be funded, an event to be scheduled, or a connection between nodes of a network. Each $a\in \A$ has an associated weight\footnote{Although  Nehring and Pivato~\cite{nehring2018median} assume real-valued weights, integer weights allow us to use compact representation languages for the constraints; a generalisation to real-valued weights is straightforward.} $w_a \in \mathbb{Z}$, which can represent a project's cost or an event's duration.
The weight vector $W$ summarises the weights of all the $m$ items in the agenda~$\A$. The notation $X(a)$ refers to the entry of vector $X$ for issue $a$. 

For an agenda $\A$ and an agent $i \in \N$, an agent's ballot is a vector $\B_i \in \{0,1\}^m$. 
 The collection of the agents' ballots is a profile $\Prof = (\B_1, \dots, \B_n)$. 
Constraints can be imposed on a weighted judgment aggregation problem, either on the collective outcome (e.g., abiding by a budget constraint), or on the individual ballots (e.g., approving a minimal number of items). 
Following Endriss~\cite{endriss2018judgment} we call the former \emph{feasibility} constraints and  the latter \emph{rationality} constraints.  $\RatB$ denotes the set of all ballots satisfying the rationality constraints, whereas $\FeasB$ is the set of outcomes satisfying the feasibility constraints. 
In what follows, we assume that constraints are expressed compactly as a set of linear inequalities. 

In the applications this paper is concerned about, we consider $\B_i(a)=1$ to be interpreted as the approval of $a$, while $\B_i(a) = 0$ is interpreted as a non-approval or an abstention on $a$ (rather than a disapproval). For instance, in participatory budgeting, a voter not approving a project often means that they are indifferent to it.
This differs from the usual interpretation in judgment aggregation where a $0$ is interpreted as a vote against an issue, which might be one of the reasons why judgment aggregation was not considered as an appropriate setting to model agents expressing approval ballots, although it is mostly a conceptual difference. 


%

\subsection{Collective Discrete Optimisation Rules}\label{sec:CDOrules}

A collective discrete optimisation (CDO) rule $\mathcal{R} : \RatB^n \rightarrow 2^\FeasB \backslash\{\emptyset\}$ is a function that associates every rational profile of individual ballots $\Prof$ with a set of feasible outcomes. Each outcome $X\in \{0,1\}^m$ is a vector denoting which items of $\A$ are accepted and which are rejected by $\mathcal{R}$ on $\Prof$. 

We now introduce three families of CDO-rules, following a modular approach. 
%
Each CDO-rule is composed of a set scoring function (or simply, a \emph{set scoring}, as they have been called by Dietrich~\cite{Dietrich2014}) and an operator. The set scorings are measures of an agent's satisfaction, comparing their individual ballot to a possible outcome.
\subsubsection{Set Scorings.}\label{sec:scoring}
A \emph{set scoring} function (or, a \emph{set scoring}) returns an agent's satisfaction, given their ballot, with respect to a set of agenda items (for instance, a set of accepted projects); i.e., it is a function $\textbf{s}: \{0,1\}^m \times \RatB \rightarrow \mathbb{R}$.\footnote{Our definitions align with the work of Dietrich~\cite{Dietrich2014} in non-weighted judgment aggregation.}
By slightly abusing set notation (when $A$ and $B$ are vectors  of the same length), we denote by $|A\cap B|$ the number of entries such that $A(x)=B(x)=1$, while $|A\setminus B|$ is the number of entries such that $A(x)=1$ and $B(x)=0$.

Given an agenda item $a$, a ballot $\B$, and a potential outcome $C$, the $\textbf{\simple}$ set scoring is defined as $\textbf{\simple}_\B(C)=\sum\limits_{a \in \A} C(a) \cdot \B(a)= |B \cap C|$. Note that this set scoring does not pay attention to weights. Its weighted generalisation $\textbf{\weight}$  uses the weight of the accepted issues to measure the satisfaction of a voter with respect to their ballot: $ \textbf{\weight}_\Prof(C)= \sum\limits_{a \in \A} C(a) \cdot \B(a) \cdot w_a$.


The $\swap$ set scoring is inspired from collective scheduling \citep{PascualEtAlAAMAS2018}. It detracts a point for every item which is approved in the agent's ballot~$\B$ but not in the candidate outcome~$C$, i.e., 
  $\swap_\B(C)= \sum\limits_{a\in A} -(1-C(a)) \cdot \B(a) = - |B \setminus C|.$ 
Then, we define $\wswap$ as a weighted version of $\swap$: 
$\wswap_\B(C)= \sum\limits_{a\in A} -w_a \cdot (1-C(a)) \cdot \B(a).$ The $\wswap$ set scoring is reminiscent of the \emph{tardy scoring} by Pascual et al.~\cite{PascualEtAlAAMAS2018}. It differs from the cost rules by Goel et al.~\cite{goel2019knapsack}, as we only look at all approved items being in $C$, whereas they count all swaps needed to turn a ballot into the outcome.


The final set scoring is based on the Chamberlin-Courant rule for approval ballots \citep{chamberlin1983representative,skowron2017chamberlin}, also known as the binary satisfaction measure \citep{talmon2019framework}. A score of $1$ is given if there is at least one item that is approved by both the agent's ballot $B$ and the candidate outcome $C$: namely, 
 $\CC_\B(C)= 1$ if $|B \cap C| \neq 0$, and $ \CC_\B(C)= 0$ otherwise.

 \begin{example}\label{ex:scorings}
 Take an agenda $A = \{a_1, \ldots, a_5\}$ with weights $W=(1,2,3,4,5)$ and let $i \in \N$ be a voter who supports $a_1$, $a_3$ and $a_5$, i.e., $B_i = (1,0,1,0,1)$. For a candidate outcome $C = (1,1,0,0,1)$ the set scorings are as follows:
 \begin{center}
 \begin{tabular}{c|ccccc|r}
 & $a_1$ & $a_2$ & $a_3$ & $a_4$ & $a_5$ & \\
 \hline
 $B_i$ & $1$ & $0$ & $1$ & $0$ & $1$ \\
 $C$ & $1$ & $1$ & $0$ & $0$ & $ 1$ \\
  \hline
 $\textbf{\simple}_{B_i}(C)$ & $+1$ & &&  & $+1$ & $2$ \\
  $\textbf{\swap}_{B_i}(C)$ &  & &$-1$&  &  & $-1$ \\
  $\textbf{\weight}_{B_i}(C)$ & $+1$ &&&& $+5$ & $6$ \\
  $\wswap_{B_i}(C)$ &&&$-3$&&& $-3$ \\ 
  $\CC_{B_i}(C)$ & $\checkmark$ &&&& $\checkmark$ & $1$\\
 \end{tabular}
 \end{center}

 We have $\textbf{\simple}_{B_i}(C) = |\{a_1, a_5\}| = 2$, as $a_1$ and $a_5$ are both supported by $i$ and in the candidate outcome $C$, and $\swap_{B_i}(C) = -|\{a_3\}| = -1$, as $a_3$ is approved by $i$  and not by $C$. For the weighted versions we multiply the item's contribution to the score by its respective weight, that is, $\textbf{\weight}_{B_i}(C) = w_{a_1} + w_{a_5} = 6$ and $\wswap_{B_i}(C) = -w_{a_3} = -3$. Finally, for the Chamberlin-Courant we have $\CC_{B_i}(C) = 1$, as there is at least one item in $C$ that $i$ supports, i.e., $|\{a_1, a_5\}| \neq 0$.
 \end{example}
 

\subsubsection{Operators.}\label{sec:CDO-op}

We define three operators which, combined with the set scorings defined above, give us the CDO-rules studied in this paper.
We take inspiration from three judgment aggregation rules: the median rule 
~\cite{nehring2018median} (generalising the Kemeny rule in preference aggregation
); the egalitarian rule 
~\cite{botan2021egalitarian} (i.e., the $d_H$-max rule in~\cite{LangPSTV15}); and the ranked agenda rule 
\cite{LangPSTV15,NehringP19,EveraereKM14}, generalising the ranked pairs rule in preference aggregation.
%
%

\paragraph{Sum Operator.}

Recalling that $\FeasB$ is the set of feasible outcomes, we define our first operator as such:

\begin{definition}[Sum Rules]\label{def:maxagg}
The sum operator $sum$, for a set scoring $\textbf{s}$, defines the following class of rules:
	$\R{sum}{\textbf{s}}(\Prof) = \arg\max\limits_{C\in \FeasB}(\sum\limits_{i \in \N}\textbf{s}_{B_i}(C)).$
\end{definition}



\paragraph{Egalitarian Operator.}

In some situations it is natural to measure the overall satisfaction by looking at 
the least satisfied agent. For example, if a group of friends are planning a trip, the dissatisfaction of a single agent might affect everyone negatively.
We thus define a class of rules induced by a maxi-min operator:

\begin{definition}[Egalitarian Rules]
	The egalitarian operator $egal$, paired with a set scoring $\textbf{s}$, defines the following rules:
	$\R{egal}{\textbf{s}}(\Prof) = \arg\max\limits_{C\in \FeasB} \min_{i \in \N}\textbf{s}_{B_i}(C).$
\end{definition}


\paragraph{Ranked Operator.} 
An intuitive way of taking collective decisions on  multiple issues is to iteratively add items to the outcome in descending order of support, 
only discarding an item if its addition would violate the constraints. Formally, we define:

\begin{definition}[Ranked Rules]
 The outcome of the $\R{rank}{\textbf{s}}(\Prof)$ rules for set scoring $\textbf{s}$ is given by 
 Algorithm~\ref{alg:ranked}.
\end{definition}

\begin{algorithm}[tb]
\caption{Algorithm for $\R{rank}{\textbf{s}}(\Prof)$ rules}
\label{alg:ranked}
\textbf{Input}: $\FeasB$, $\Prof$, $A$\\
\textbf{Output}: $\R{rank}{\textbf{s}}(\Prof)=X$
  \begin{algorithmic}[1]
    \STATE{ $S:=\emptyset$ and $X:=\{0\}^A$}
      \WHILE{$S\neq A$}
        \STATE $x:= \arg\max\limits_{a\in A\backslash S}\textbf{s}_\Prof(X_{+a}) $ 
       \IF{$X_{+x\upharpoonright (S \cup \{x\})} \in \FeasB_{\upharpoonright (S \cup \{x\})}$}
         \STATE$X(x)=1$ 
         \ENDIF
      \STATE $S\leftarrow S\cup\{x\}$
      \ENDWHILE
    \STATE \textbf{return} $\R{rank}{\textbf{s}}(\Prof)=X$   
  \end{algorithmic}
  \end{algorithm} 

Let $X_{\upharpoonright S}$ be the restriction of vector $X$ to the elements of set $S$, and let $X_{+a}$ be $X$ with $X(a)=1$. Algorithm~\ref{alg:ranked} initiates $X$ to a vector of rejections for every item of the agenda and $S$ to an empty set that will keep track of which items have been considered so far. 
The algorithm finds an item $x\in A\setminus S$ whose approval maximises the score of $X_{+x}$ and checks if the addition of $x$, when restricted to the items which have already been considered plus $x$, i.e., $X_{+x\upharpoonright (S \cup \{x\})}$, is extendable to a feasible outcome, i.e., $X_{+x\upharpoonright (S \cup \{x\})} \in \FeasB_{\upharpoonright (S \cup \{x\})}$. If it is, $x$ is accepted in the outcome $X$, letting $X(x)=1$. Then $x$ is added to $S$. The algorithm returns $\R{rank}{\textbf{s}}(\Prof)=X$ when $S=A$. 

\begin{example}\label{ex:greedy}
Take an agenda $A=\{a_1,a_2,a_3,a_4\}$ with $\FeasB = \{(1,0,0,1),(1,1,0,0),\allowbreak (0,1,1,0)\}$.
Let $\Prof$ be a profile with four voters where each $a_i$ is approved by exactly the first $i$ voters. We do not specify $W$ as we only consider $\textbf{\simple}$ (and $\swap$) set scorings.

At each step, $\R{rank}{\textbf{s}}$ selects the 
project with the highest score and adds it, unless it leads to an infeasible outcome.
The outcome of $\R{rank}{\textbf{\simple}}(\Prof)$ is $\{(1,0,0,1)\}$, where $a_4$ is added first, $a_1$ is added last. The other items are skipped as they are not feasible.

On the other hand,  $\R{sum}{\textbf{s}}$ selects those feasible outcomes which maximise the sum of the voters' scores. In our example, $\R{sum}{\textbf{\simple}}(\Prof) = \{(1,0,0,1),(0,1,1,0)\}$, as both yield an overall score of five, whereas $(1,1,0,0)$ yields a score of three. 

Finally, $\R{egal}{\textbf{s}}$ selects feasible outcomes which maximise the satisfaction of the least satisfied voter. Here we have $\R{egal}{\textbf{\simple}}(\Prof) = \{(1,0,0,1)\}$, as it is the only outcome where the least satisfied voter gets one item she approves. In contrast, $\R{egal}{\textbf{\swap}}(\Prof) = \FeasB$, as the first voter approves every item and in each feasible outcome there are exactly two items (that she likes) which are not approved. Thus, $\R{egal}{\textbf{\simple}}$ and $\R{egal}{\textbf{\swap}}$ can differ.
\end{example}


 
There is a plethora of other rules defined for specific problems which could be modeled as CDO-rules: one such example is the proportional greedy rule for participatory budgeting
from Talmon and Faliszewski~\cite{talmon2019framework}---we do not include it here, as the rule is harder to motivate for other settings. In Section~\ref{sec:equivja} we will show how some of our CDO rules coincide and that some are equivalent to judgment aggregation rules.



\subsection{Equivalence Results in Judgment Aggregation}\label{sec:equivja}

In this section we present 
results showing the equivalence of some of our rules with known judgment aggregation rules, 
where two rules are \emph{equivalent} when we can show that they are both instances of each other. 
We say that a rule $\R{}{}$ \emph{is an instance of} rule~$\R{'}{}$ if there exists a translation from the setting of $\R{}{}$ to the setting of $\R{'}{}$ such that the two rules give equivalent outcomes. 

\begin{lemma}\label{lem:equivalences}
The following equivalences hold:
\begin{enumerate}[(i)]
    \item The median rule  \citep{nehring2018median}, $\R{sum}{\textbf{\simple}}$, and $\R{sum}{\swap}$ are all equivalent;   
    \item $\R{sum}{\textbf{\weight}}$, $\R{sum}{\wswap}$, and the weighted median rule \citep{nehring2018median} are all equivalent;
    \item $\R{rank}{\textbf{\simple}}$ is equivalent to the ranked agenda rule in judgment aggregation~\cite{lang2013judgment}. 
    \end{enumerate}
\end{lemma}
Due to these equivalences we will now refer to our CDO-rules by their names in judgment aggregation, e.g., $\R{sum}{\textbf{\simple}}$ will be referred to as the median rule.

\begin{proof}[sketch]
For statement (i), the equivalence between $\R{sum}{\textbf{\simple}}$ and $\R{sum}{\swap}$  is clear given the definitions of the rules. 
For equivalence with the median rule, we show that the median rule is an instance of $\R{sum}{\textbf{\simple}}$ and vice versa. For the left-to-right direction, we translate the judgment aggregation framework by  Nehring and Pivato~\cite{nehring2018median} to our setting. We define two agenda items for each judgment aggregation issue (one corresponding to the acceptance and another for the rejection of the issue).
The respective outcomes are equivalent, as both rules maximise an overall score, which only differs by constant factors. 
For the right-to-left direction, we translate our setting to the same framework of judgment aggregation, where every agenda item $a$ in our setting is translated into two judgment aggregation issues, $a^s$ and an auxiliary issue $a^*$ to bias the outcome towards acceptance of $a^s$. This can be done by imposing constraints that the output abides by $a^s \leftrightarrow a^*$ and that every ballot approves $a^*$. Analogously to the other direction, the sum over all issues for a candidate outcome coincides with the simple set scoring (modulo some constant factors, which can be ignored while maximizing).

The proof of (ii) follows the same structure as (i), only that we impose weights on every issue in the translation. The constant factors that the rules differ by changes, i.e., in both directions both rules now use the total weight of all items.

For (iii) the translation used in the direction left-to-right of (i) also holds. For the other direction we use a different translation, where an additional $n+1$ voters, who approve all items of the agenda, are added to the profile.
\end{proof}

%

 One slightly surprising consequence of Lemma~\ref{lem:equivalences} is that two pairs of rules with seemingly completely different measures of voters' satisfaction (namely the pair $\textbf{\simple}$ and $\swap$, and the pair $\textbf{\weight}$ and $\wswap$), give the same outcomes when paired with the $sum$ operator. Moreover, by Lemma~\ref{lem:equivalences} we have that the axiomatic characterisations for the median and the weighted median rules given by Nehring and Pivato~\cite{nehring2018median} also hold for $\R{sum}{\textbf{\simple}}$ and $\R{sum}{\swap}$, and $\R{sum}{\textbf{\weight}}$ and $\R{sum}{\wswap}$, respectively. Observe however that there is no such equivalence between the egalitarian median rule,  $\R{egal}{\textbf{\simple}}$ and $\R{egal}{\swap}$---as we have seen in Example~\ref{ex:greedy} that the two CDO-rules do not coincide.



The first two statements of Lemma~\ref{lem:equivalences} show that our setting and the one by Nehring and Pivato~\cite{nehring2018median} can be used interchangeably for some rules. Our choice of model is motivated by rephrasing their model to act as a clear intermediary step between different CDO domains---whereas standard judgment aggregation cannot show this clear link. We will prove in Section~\ref{sec:comparing} that domain-specific CDO-rules are instances of known judgment aggregation rules: this step which seems intuitive in our unifying model, would not be as obvious in judgment aggregation. Although we highlight only a handful of CDO-problems, which use judgment aggregation rules, our model could act as a stepping stone to unfold connections to countless other CDO-rules more naturally.

One reason why our model is more clearly connected to typical CDO-problems is its biased domain $\{0,1\}$, instead of $\{-1,1\}$. Furthermore, given the modular setup of the rules, our model is well-equipped to measure an individual's satisfaction (whereas Nehring and Pivato~\cite{nehring2018median} group voters' ballots together, removing the possibility of non-additive measures of satisfaction) and to get natural extensions of judgment aggregation rules by altering the set scorings (depending on the specific needs of a CDO problem).
One could reduce the weighted median rule to its unweighted version by introducing $w_i$ copies of an item of weight $w_i$ and adding constraints making these copies equivalent. However, this translation would be exponentially large, and would not generalise to all rules (it notably fails for ranked rules).


\section{Collective Discrete Optimisation}\label{sec:CDO}

In this section we show how a number of seemingly unrelated problems (collective selection, collective scheduling, and collective network design) can be phrased in the setting of weighted judgment aggregation described in Section~\ref{sec:model}. 
We then show that the rules studied in each specific setting are related to one another, since they can be interpreted as instances of the CDO-rules we introduced in Section~\ref{sec:CDOrules}.

\subsection{CDO Problems}\label{sec:agendaconstraint}

We explain in this section how to embed a variety of CDO problems into our setting.

\paragraph{Collective Selection (Participatory Budgeting).}
Agents select projects to be funded by a given resource: e.g., in participatory budgeting agents approve projects to be funded by a limited budget.
The \emph{selection agendas} $\A=\{a_1, \dots, a_m\}$ contain items representing a set of $m$ projects, such that the weight $w_{a_p}$ associated with item $a_p$ represents the cost of project $p$.   
The weight of the collectively selected projects must not exceed the budget limit $\ell \in \mathbb{N} $. Hence, an outcome $X\in \FeasB$ has to satisfy the following equation:
$ \sum_{a\in A} X(a) \cdot w_a\leq \ell$.    
Additional feasibility constraints could also specify, for instance, that a percentage of the budget must be spent on projects in a given area.\footnote{E.g., part of the budget for the PB campaign in Paris was reserved to low-income districts \citep{cabannes2017participatory}.}

If there are no rationality constraints, $\RatB=\{0,1\}^m$, agents can approve any number of items (regardless of their costs); if $\RatB=\FeasB$, agents must submit their ideal allocation (under the budget limit). 
%
Other collective selection problems include collective knapsack or knapsack voting \citep{goel2019knapsack}, where rationality and feasibility constraints coincide, and weighted committee selection \citep{klamler2012committee}.


\paragraph{Collective Networking.} In collective networking, 
agents have to collectively decide on how to design a communal network---be it water pipelines, internet services, or travel connections between countries. The agents specify which connections they approve of in a network, and the goal is to find a spanning tree from such input, i.e., an undirected acyclic graph that includes all nodes, maximising the satisfaction of the agents. 
This problem has been introduced and studied by Darmann et al.~\cite{DarmannEtAlCOMSOC2008,darmann2009maximizing}, building on the vast literature on spanning trees in both computer science and economics.

Given an undirected network $G=(V,E)$ a \emph{networking agenda} is the set of items $A=\{a_{ij}\mid (i,j)\in E\}$, where $w_{a_{ij}}$ is the cost of adding edge $(i,j)$ to the outcome.  Darmann et al.~\cite{darmann2009maximizing} consider weighted edges but no budget limit, assuming that the central authority will fund any outcome.

The feasibility constraints require that the set of accepted edges forms an acyclic connected graph (note that a budget limit can also be imposed). These constraints can be formulated in ILP in many ways; we here focus on the single commodity flow model by Abdelmaguid~\cite{abdelmaguid2018efficient}.
Using this formulation we first move from undirected to directed graphs, and we then forget the direction of the edges to obtain the collective spanning tree.
There are $|E|$ variables $a_{ij}$ stating whether $(i,j)$ is in the spanning tree, and $2|E|$ variables $y_{ij}$ and $y_{ji}$ in set $Y$ for the two directions of each edge in $E$. Each $y_{ij}\in \mathbb{N}$ describes the flow going from node $i$ to node $j$. 
We have $|V|$ constraints of the following form, for $j\in V$:
\begin{equation*}\label{eq:spantree1}
{\sum}_{i : (i,j)\in E} (y_{ij}-y_{ji})=\begin{cases} 1-|V|, & \text{ if } j=1\\
 1, & \text{ otherwise}
\end{cases}
\end{equation*}

The first case accounts for the (artificial) root of the tree $j=1$, having no in-flowing edges. Thus, $y_{i1} = 0$ for all $(i, 1) \in E$ and 
the out-flowing edges total a weight of $|V|-1$. The second case ensures that in a spanning tree the in-flowing weight exceeds the out-flowing by one.
The next constraints ensure that directed edges correspond to the undirected edges:
    $y_{ij}\leq (|V|-1)x_{ij} \text{ and } y_{ji}\leq (|V|-1)x_{ij}.$
 For every $(i,j)\in E$, the constraints stipulate that each of $y_{ij}$ and $y_{ji}$ can carry flow, only when $x_{ij}$ is in the spanning tree.
 Finally, we ensure that the tree has $|V|-1$ edges: 
     $\sum\limits_{(i,j)\in E}x_{ij}= |V|-1.$
 
 \begin{example}[Collective Spanning Tree]\label{example:spanning}
Four cities have funding to improve their train connections, and need to decide which rails should become high-speed. Let $G = (V, E)$ represent the cities and the candidate rails, where $V = \{H,I,J,K\}$ and $E = \{(H,I),\allowbreak (H,K), (I,J), (I,K), (J,K)\}$. Thus, the agenda is $\A = \{a_{HI}, a_{HK}, a_{IJ}, a_{IK}, a_{JK}\}$, and $W = (1, 2, 4, 3, 2)$ is the monetary cost, in millions of euros. 
City $H$ may be interested in a faster connection with cities $I$ and $K$, 
 thus submitting ballot $B_H = (1,1,0,0,0)$.
A solution would need to select which connections to improve, making sure that all cities are connected by high-speed rails.
 \end{example}

\paragraph{Collective Scheduling.} In collective scheduling the agents submit transitive and asymmetric orderings over a set of jobs to be performed on a single machine, indicating their preferred order of execution, and a collective scheduling rule decides on the execution order of the jobs. Pascual et al.~\cite{PascualEtAlAAMAS2018} assume that the output schedule has no gaps and is complete: hence the setting is equivalent to the aggregation of orderings with weighted alternatives, where the weights correspond to the duration of each job.
If $P=\{p_1, \dots, p_m\}$ is a set of jobs to be scheduled, with execution times $t_x$ for job $p_x$, a \emph{scheduling agenda} is $\A=\{a_{x \prec y } \mid p_x \neq p_y \in P\}\cup \{a_{0 \prec x} \mid p_x \in P\}$, where the approval of $a_{x \prec y}$ shows support for $p_x$ being scheduled before $p_y$, and the approval of $a_{0 \prec x}$ means that $p_x$ should be scheduled first. Given an item $a_{x \prec y}$, its weight corresponds to $w_{p_y}$, i.e., the duration $t_y$ of the second scheduled job. In this way, the duration of the first job $p_x$ is captured by $w_{a_{0\prec x}}$, and accordingly for all subsequent jobs.

The outcome $X$ must be a linear order of the jobs, which can be easily formulated in ILP via feasibility constraints imposing transitivity and asymmetry of scheduled jobs, as well as a constraint stating that there is a unique starting job: $ \sum_{j\in [1,m]}X(a_{0\prec j}) =1$.

 
 Pascual et al.~\cite{PascualEtAlAAMAS2018} assume that $\RatB=\FeasB$, but in our setting weaker rationality constraints can be considered, such as agents submitting partial orders over the jobs.

\begin{example}[Program Schedule]\label{Ex:schedballot}
A faculty is scheduling the mandatory courses for the first year mathematics students. The faculty members have to decide the most beneficial ordering of courses for students. The courses are $\P=\{p_1,p_2,p_3,p_4\}$, with agenda $\A= \{ a_{0\prec 1}$,$ a_{0\prec 2},$ $a_{0\prec 3}$, $a_{0\prec 4}$, $a_{1\prec 2}$,$\ldots \}$. Professor $i$ believes that course $p_2$ should come first, that course $p_4$ should come last, and has no preference over the ordering between $p_1$ and $p_3$. Her ballot is thus $\B_i=(0,1,0,0,0, 0, 1, 1, 1, 1, 0, 0,1, 0, 0, 0)$. 
\end{example}


\subsection{CDO Rules in Judgment Aggregation}\label{sec:comparing}

We now show that each of the rules defined in the specific settings of participatory budgeting, collective scheduling, and collective networking, are all instances of the general rules we defined in Section~\ref{sec:CDOrules}. 
Recall that a rule $\R{}{}$ is an instance of rule $\R{'}{}$ if there exists a translation from the setting of $\R{}{}$ to the setting of $\R{'}{}$ such that the two rules give equivalent outcomes. 
In the following proposition we prove that three known rules from different settings are instances of the median rule ($\R{sum}{\textbf{\simple}}$).

\begin{proposition}\label{prop:MedianSameMAX}
The following rules are all instances of the median rule: 
\begin{itemize}
    \item the max rule with cardinality satisfaction measure by Talmon and Faliszewski~\cite{talmon2019framework},
    \item the maximum collective spanning tree by Darmann et al.~\cite{DarmannEtAlCOMSOC2008},
    \item the utilitarian aggregation with swap distance by Pascual et al.~\cite{PascualEtAlAAMAS2018}. 
\end{itemize}  
\end{proposition}

\begin{proof}[sketch]
For each of these specific CDO domains we use the translation given in Section~\ref{sec:agendaconstraint}. Given  Lemma~\ref{lem:equivalences}, we show that the max rule with cardinality satisfaction measure by Talmon and Faliszewski~\cite{talmon2019framework} is an instance of $\R{sum}{\textbf{\simple}}$ as it is equivalent to the median rule. Both rules maximise the satisfaction of voters by counting the number of items approved in the candidate outcome and the ballots. Furthermore, the set of feasible outcomes is equivalent to those candidate outcomes that do not surpass the budget limit. The same holds for the maximum collective spanning tree by Darmann et al.~\cite{DarmannEtAlCOMSOC2008} given its translation, as feasible outcomes correspond to spanning trees. 

 Pascual et al.~\cite{PascualEtAlAAMAS2018} showed that their utilitarian aggregation rule with swap distance is equivalent to the Kemeny rule in preference aggregation, which is known to be equivalent to the median rule in judgment aggregation. 
\end{proof}

Further, we show that rules from CDO problems are instances of $\R{sum}{\CC}$.

\begin{proposition}\label{prop:maxcc}
 The Chamberlin-Courant rule for approval ballots  by  Skowron and Faliszewski~\cite{skowron2017chamberlin} and its generalisation to participatory budgeting by Talmon and Faliszewski~\cite{talmon2019framework} are both instances of $\R{sum}{\textbf{\CC}}$.
\end{proposition}

\begin{proof}[sketch]
Talmon and Faliszewski~\cite{talmon2019framework} propose a rule which outputs a set of items maximising the number of ballots with nonempty intersection with the outcome (a special case is the multi-winner variant by Skowron and Faliszewski~\cite{skowron2017chamberlin}).  Both are modeled by $\R{sum}{\CC}$, where the sum operator is paired with a set scoring giving each voter a score of one if her ballot shares an item with the outcome (and zero otherwise).
\end{proof}


We now show that four rules introduced in different settings are instances of $\R{egal}{\textbf{\simple}}$.

\begin{proposition}\label{prop:EgalInstanceRules}
The following rules are instances of $\R{egal}{\textbf{\simple}}$: 
\begin{itemize}
    \item the egalitarian median rule in judgment aggregation ($d_H$-max by Lang et al.~\cite{LangPSTV15}, or \textit{MaxHam} by Botan et al.~\cite{botan2021egalitarian})
    \item the egalitarian aggregation function with swap cost by Pascual et al.~\cite{PascualEtAlAAMAS2018},
    \item the maximin voter satisfaction problem for approval voting for spanning trees by Darmann et al.~\cite{darmann2009maximizing},
    \item the maximin approval voting rule for multiwinner elections by Brams et al.~\cite{brams2007minimax}.
\end{itemize}
\end{proposition}

\begin{proof}[sketch]
The egalitarian median rule is an instance of $\R{egal}{\textbf{\simple}}$ as the functions for both rules are equivalent, given the translation from judgment aggregation to our setting (described in the proof sketch of statement (i) in Lemma~\ref{lem:equivalences}). The egalitarian aggregation function with swap costs is an instance of $\R{egal}{\textbf{\simple}}$ using the translation in Section~\ref{sec:agendaconstraint}: the rational scheduling ballots are such that each must contains the same number of approvals. For example, an agent must accept exactly one of $a_{p\prec q} $ and $a_{q\prec p}$. Subsequently, the swap cost for an agent's ballot with respect to a candidate outcome differs from the simple set scoring  by the same constant factor for each agent. Therefore, maximizing the minimum voters satisfaction yields the same result after translating. The latter rules follows by definition.
\end{proof}



Next we relate our ranked rules with known participatory budgeting rules. 

\begin{proposition}\label{prop:RankPairsSameRANK}
The greedy rule with cardinality satisfaction measure by Talmon and Faliszewski~\cite{talmon2019framework} is an instance of the ranked agenda rule ($\R{rank}{\textbf{\simple}}$), and the greedy rule with the generalised Chamberlin-Courant satisfaction measure by Talmon and Faliszewski~\cite{talmon2019framework} is an instance of $\R{rank}{\CC}$.
\end{proposition}

\begin{proof}[sketch]
It is easy to see that the \emph{rank} operator behaves in the same way as their greedy rule, adding the item which increases the score (satisfaction measure) the most while not breaking the feasibility constraints (surpassing the budget limit). Furthermore, the cardinality satisfaction measure and the $\textbf{\simple}$ set scoring give the same outcome for every input. Therefore, the greedy rule with cardinality satisfaction measure by  Talmon and Faliszewski~\cite{talmon2019framework} is an instance of $\R{rank}{\textbf{\simple}}$. Moreover, since $\CC$ and the Chamberlin-Courant satisfaction measure assign the same score in the same situations, the last two rules in the statement will also give the same outcomes. 
\end{proof}




\section{Implementing and Testing CDO-Rules}\label{sec:ILP}

The general framework we propose can not only be used for theoretical comparisons of specific rules, but also to get a modular implementation of CDO-rules, where by simply plugging in the constraints one can focus on a particular application.
We report here on experiments that compare the behaviour of three CDO-rules on the setting of collective networking. We use integer linear programming (ILP) for our implementation, a successful formalism complemented by a vast literature and efficient solvers.\footnote{To the best of our knowledge, this is the first ILP formulation for judgment aggregation rules; the closest work being the one by de Haan and Slavkovik \cite{DeHaanSlavkovikIJCAI2019} via answer set programming.} 

\subsection{Experimental Design} \label{sec:exprisetup}

We focus on the collective networking problem (defined in Section~\ref{sec:agendaconstraint}) without a budget constraint, comparing the processing time of three rules: $\R{sum}{\textbf{\simple}}, \R{sum}{\CC}$, and $\R{egal}{\textbf{\simple}}$.
We do not study $\R{rank}{\textbf{s}}$ rules, as they are solvable in polynomial time (if the constraints can be checked in polynomial time).
Our implementation uses the open-source \emph{GNU Octave} software \citep{OCTAVE}, and its standard ILP solver \texttt{glpk}, using two-phase primal simplex method.
We emphasise that our implementation is modular and  can thus be easily adapted to account for any collective discrete optimisation problem.



In collective networking, agents vote on the edges $E$ of a connected network and the CDO-rule finds a collective spanning tree.
We generate 49 connected graphs $G=(V,E)$ with number of nodes varying between $6$ and $8$, i.e., $|V|\in [6,8]$. 
For each value of $|V|$ we generate connected graphs with $|E|\in \left[|V|-1, \frac{1}{2}|V|(|V|-1)\right]$: i.e., from trees to complete graphs.
We randomly generate 49 graphs, one for each pair $(|V|,|E|)$, as follows: 
starting with $|V|$ nodes, we mark one node as connected; while the graph is not connected, we randomly add an edge between a connected and an unconnected node, and we then add in the remaining $|E|-|V|+1$ edges. 


We let $|\N|=100$, and on each generated graph $G$ we create $10$ base profiles. 
 Each base profile $bp$ is an $n\times |E|$ matrix, where for each $i\in \N$ we have that $bp_i\in (0,1]^E$: 
 each item of the agenda is assigned a real number between $0$ and $1$ to represent the acceptance rate of an issue by an agent. Each base profile is then transformed into $9$ new profiles following a variant of the $p$-impartial culture model of Bredereck et al.~\cite{bredereck2019experimental}. 
According to this model, when generating approval voting ballots one can assume that every agent independently approves each item of the agenda with probability $p$. Thus, for each base profile 
we create one profile for every $p\in \{0.1,\ldots,0.9\}$
, such that voter $i$ approves an edge $e\in E$ if and only if the entry for $e$ in the base profile is at most $p$.

 \subsection{Experimental Results}\label{sec:experi}
We ran $\R{sum}{\textbf{\simple}}, \R{sum}{\CC}$, and $\R{egal}{\textbf{\simple}}$ on $4410$ instances ($49$ graphs, paired with $10$ base profiles and $9$ levels of $p$) comparing their processing times to compute an outcome.\footnote{Six instances were computed in parallel on an Intel i7 processor at 4.2 GHz with 4 physical and 8 logical cores and 32 GB of memory.} 
 %
 Figure~\ref{fig:PandVvary} shows, for $p\in \{0.2, 0.8\}$, the mean processing times over all profiles and all generated networks 
 with $|V|\in\{6,7,8\}$, using a log$_2$-scale on the $y$-axis. 

   \begin{figure}[t!]
     \centering
     \subfigure{\includegraphics[width=0.45\linewidth]{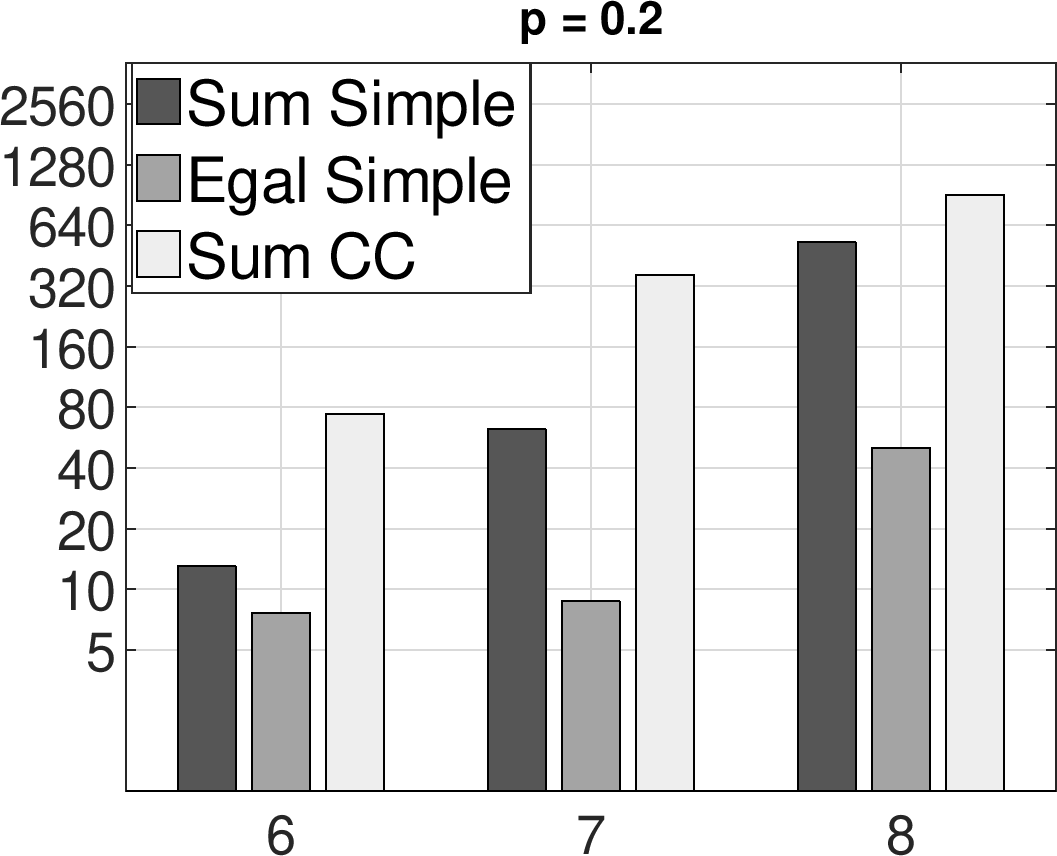}}\hfil
     \subfigure{\includegraphics[width=0.45\linewidth]{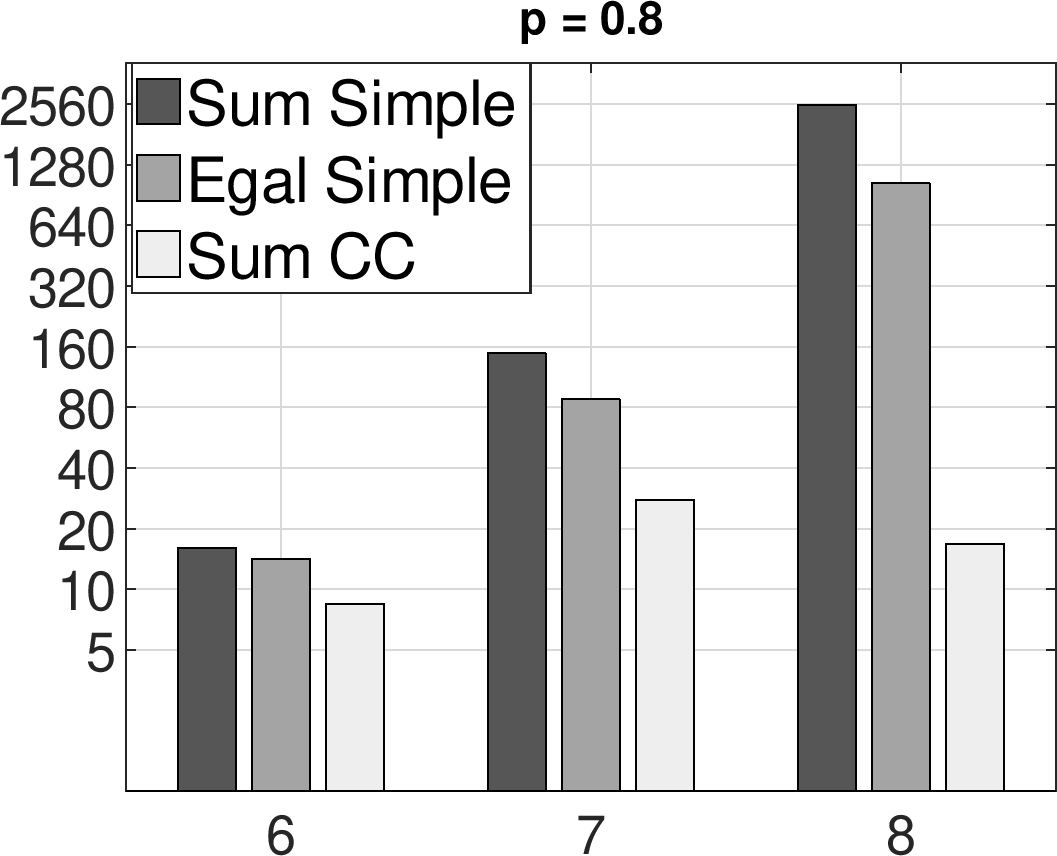}}
    \vspace{-3mm}
     \caption{Mean processing time for the $\R{sum}{\textbf{\simple}}, \R{egal}{\textbf{\simple}}$ and $\R{sum}{\CC}$ rules applied on the spanning tree problem with 
     $p\in \{0.2, 0.8\}$. The $x$ axis represents the number of nodes in the graph, 
     while the $y$ axis represents the mean processing time (milliseconds) on a 
     $\log_2$-scale.}
     \label{fig:PandVvary}
     \vspace{-3mm}
 \end{figure}

 We observe that the run-time of $\R{sum}{\CC}$ is inversely proportional to the acceptance level $p$, 
 confirming the intuition that finding a collective spanning tree with $\R{sum}{\CC}$ is more difficult with sparse ballots.
 Overall, $\R{sum}{\textbf{\simple}}$  is slower than the other two rules 
 (except for $\R{sum}{\CC}$ with small values for $p$). 
 Note that without additional budget constraints, the $\R{sum}{\textbf{\simple}}$ rule is equivalent to finding the maximum spanning tree where the edge's approvals are the weights. 
Finally, the run-time of $\R{egal}{\textbf{\simple}}$ increases steadily with the number of nodes for small values of $p$ and more rapidly for larger values. 
 This can be explained 
 by analysing the ILP formulation of $\R{egal}{\textbf{\simple}}$: it maximises the minimum score of any agent, which is bounded by the minimum number of items that any agents has approved and this bound is low when the acceptance of $p$ is low. 

Experiments for $|V| \geq 9$ resulted in some time-outs after $1200$ seconds. For $\R{sum}{\textbf{\simple}}$, these instances were almost  complete graphs paired with high $p$ values. For $\R{sum}{\CC}$ and $\R{egal}{\textbf{\simple}}$, the pattern is more complex and depends on graph and profile structures.

\section{Conclusions and Future Work}\label{sec:conclu}

Our primary contribution has been to 
bridge existing proposals of CDO problems and represent them in a unified framework, i.e.,
phrasing them in weighted judgment aggregation and defining modular rules to compute the outcome.
Thanks to our model we were able to prove how numerous existing algorithms defined for specific problems (participatory budgeting, collective scheduling, and collective networking) are actually related, as they are instances of one of our general CDO rules.
Our model thus establishes connections between these CDO problems, helping us to understand and solve them better. 
%
We also showcased a modular implementation of CDO rules, presenting an experimental comparison of three rules for the setting of collective networking.


Given the vast existing research in discrete optimisation, our proposed definition of collective discrete optimisation can pave the way for further general studies of such problems.
First, many of the specific rules we generalised are compared using the axiomatic method. While most of the axiomatic properties are motivated by domain-specific desiderata, there are arguably properties that can be studied at a more general level (the axiomatisation of the weighted median rule of  Nehring and Pivato~\cite{nehring2018median} is one such example).
Second, the computational complexity of CDO-rules can be studied at the general level in weighted judgment aggregation, aiming to identify islands of tractability by suitably restricting the set of constraints for specific problems. 
Finally, making our code accessible on an online platform would allow, for instance, a town hall employee to test different participatory budgeting rules on existing data, or a researcher in networks to easily test algorithms for collective spanning tree problems.




\end{document}